\documentclass{llncs}
\usepackage{llncsdoc}
\usepackage{amsmath}%
\usepackage{amsfonts}%
\usepackage{amssymb}%
\usepackage[pdftex]{graphicx}%
\usepackage{float}%
\usepackage{subfigure}%
\usepackage{multicol}%
\usepackage{epstopdf}
\begin{document}
\title{Stopping Criterion for the Mean Shift Iterative Algorithm}
\author{Yasel Garc\'es \and Esley Torres \and Osvaldo Pereira \and Claudia P\'erez \and \\ Roberto Rogr\'iguez}
\institute{Institute of Cybernetics, Mathematics and Physics, Havana, Cuba \quad \email{ygarces@icimaf.cu},
\email{esley@icimaf.cu},
\email{rrm@icimaf.cu}}
\maketitle
\begin{abstract}
Image segmentation is a critical step in computer vision tasks constituting an essential issue for pattern recognition and visual interpretation. In this paper, we propose a new stopping criterion for the mean shift iterative algorithm by using images defined  in $\mathbb{Z}_{n}$ ring, with the goal of reaching a better segmentation. We carried out also a study on the weak and strong of equivalence classes between two images. An analysis on  the convergence with this new stopping criterion is carried out too.
\end{abstract}
\section{Introduction}
Many techniques and algorithms have been proposed for digital image segmentation. Unfortunately, traditional segmentation techniques using low-level, such as thresholding, histograms or other conventional operations are rigid methods. Automation of these classical approximations is difficult due to the complexity in shape and variability within each individual object in the image. Mean Shift (MSH) is a robust technique which has been applied in many computer vision tasks. MSH as an iterative algorithm has been used in many works by using the entropy as a stopping criterion \cite{Rodriguez11,Rodriguez11a,Rodriguez12,Rodriguez08,Dominguez11}.\\   
Entropy is an essential function in information theory and has special uses for images data, e.g., restoring images, detecting contours, segmenting images and many other applications \cite{Zhang03,Suyash06}. However, in the field of images, the range of properties of this function could be increased if the images would be defined in $\mathbb{Z}_{n}$ rings. \\   
In this paper, we compare the stability of iterative MSH algorithm using a new stopping criterion based on ring theory with respect to the stopping criterion used in \cite{Rodriguez11,Rodriguez11a,Rodriguez12,Rodriguez08}. The remainder of the paper is organized as follows: Theoretical aspects related with the entropy and the defined images in $\mathbb{Z}_{n}$ ring are exposed in Section \ref{Theorical aspect: Entropy}. Here, a special attention is dedicated to the benefits of image entropy in the $\mathbb{Z}_{n}$ ring. Section \ref{Experiments and Results} shows the experimental results, comparisons and discussion. And finally the most important conclusions are given in the last section.

\section{Theoretical Aspects: Entropy}
\label{Theorical aspect: Entropy}
Entropy is a measure of unpredictability or information content. In the space of the digital images the entropy is defined as \cite{Shannon48}.
\begin{definition}[Image Entropy]
The entropy of the image $\mathcal{A}$ is defined by
\begin{equation}
\label{entropy definition}
E(\mathcal{A})=-\sum_{x=0}^{2^{B}-1}p_{x}log_{2}{p_{x}},
\end{equation}
where B is the total quantity of bits of the digitized image $\mathcal{A}$ and $p(x)$ is the probability of occurrence of a gray-level value. By agreement $\log_{2}(0)=0$.
\end{definition}
In recent works  \cite{Rodriguez11,Rodriguez11a,Rodriguez12,Rodriguez08} the entropy is an important point to define a stopping criterion for a segmentation algorithm based on an iterative computation of the mean shift filtering. In \cite{Rodriguez11,Rodriguez11a,Rodriguez12,Rodriguez08} the stopping criterion is
\begin{equation}
\label{old criterion}
\nu(\mathcal{A},\mathcal{B})=\vert E(\mathcal{A})-E(\mathcal{B}) \vert , 
\end{equation}
where $E(\cdot)$ is the function of entropy and the algorithm is stopped when $\nu(\mathcal{A}_{k},\mathcal{A}_{k-1})\leq \epsilon$. Here $\epsilon$ and $k$ are respectively the threshold to stop the iterations and the number of iterations.
\begin{definition}[Weak equivalent in Images]
\label{weakly equivalent}
Two images $\mathcal{A}$ and $\mathcal{B}$ are weakly equivalents if
$$
E(\mathcal{A})=E(\mathcal{B}).
$$
We denote the weak equivalent between $\mathcal{A}$ and $\mathcal{B}$ using $\mathcal{A}\asymp \mathcal{B}$.
\end{definition}
Trivial implication is: 
$$
\mathcal{A}\asymp \mathcal{B}\Longleftrightarrow \nu(\mathcal{A},\mathcal{B})=0.
$$
Note that using the Definition \ref{weakly equivalent} the stopping criterion  defined in (\ref{old criterion}) is a measure to know when two images are close to be weakly equivalents.\\\
Figure \ref{different images comparation} shows two different images of $64\times 64$. A reasonable stopping criterion should present a big difference between Figure \ref{one_a} and Figure \ref{two_b}. However, by using the expression (\ref{old criterion}), we obtain that $\nu(Figure\ \ref{one_a},\ Figure\ \ref{two_b})=0$.
\begin{figure}[H]
\centering
\subfigure[]{\includegraphics[scale=2]{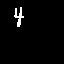}\label{one_a} }
\subfigure[]{\includegraphics[scale=2]{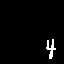}\label{two_b} }
\vspace{-0.1in}
\caption{Dissimilar Images}
\label{different images comparation}
\end{figure}
The defined stopping criterion in (\ref{old criterion}) never consider the spacial information between the images $\mathcal{A}$ and $\mathcal{B}$. For this reason, it is possible to have two very different images and to obtain a small value by using (\ref{old criterion}).\\
This is a strong reason to consider that the defined stopping criterion in (\ref{old criterion}) is not appropriate and provide instability in the iterative mean shift algorithm. For this reason, it is necessary to consider other stopping criterion that provides a better performance.\\
It is natural to think that two images are close if their subtraction is close to zero. The problem of this idea is that, in general, when the subtraction gives negative values many authors consider to truncate to zero these elements. This consideration, in general, it not describe the difference between two images, and in some cases, it is possible to lose important information.\\
For this reason, it is necessary to define a structure such that the operations between two images are intern.
\begin{definition}[$\mathbb{Z}_{n}$ Ring]
The $\mathbb{Z}_{n}$ ring is the partition of the set of integers $\mathbb{Z}$ in which the elements are related by the congruence module $n$.
\end{definition}
Mathematically speaking, we say that $a$ is in the class of $b$ ($a\in C_{b}$) if $a$ is related by $(\sim )$ with $b$, where
\begin{eqnarray*}
a\sim b &\Longleftrightarrow a \equiv b (mod\ n) \overset{def}{\Longleftrightarrow} (b-a)\in n\mathbb{Z},\quad \mbox{where}\\
n\mathbb{Z} &= \left\lbrace 0,n,2n,\ldots \right\rbrace \quad \mbox{and}\quad n\in \mathbb{Z}\quad \mbox{is fixed}. 
\end{eqnarray*}
Consequently $\mathbb{Z}_{n} = \lbrace C_{0},\ C_{1},\ldots ,\ C_{n-1}\rbrace$.\\
If we translate the structure of the $\mathbb{Z}_{n}$ ring to the set of images of size $k\times m$ where the pixel values are less that $n-1$ and we denote this set as $G_{k\times m}(\mathbb{Z}_{n})$, we obtain the next result.
\begin{theorem}
\label{theorem ring matrix}
The set $G_{k\times m}(\mathbb{Z}_{n})(+,\cdot)$, where $(+)$ and $(\cdot )$ are respectively the pixel-by-pixel sum and multiplication in $\mathbb{Z}_{n}$, has a ring structure.
\end{theorem}
\begin{proof}
As the pixels of the image are in $\mathbb{Z}_{n}$, they satisfies the ring axioms. The operation between two images was defined pixel by pixel, then is trivial that $G_{k\times m}(\mathbb{Z}_{n})$ under the operations $(+,\cdot)$ of the $\mathbb{Z}_{n}$ ring inherits the ring structure.\qed
\end{proof}
In this moment, we have an important structure where we can operate with the images. In the ring $G_{k\times m}(\mathbb{Z}_{n})(+,\cdot)$ the sum, subtraction or multiplication of two images always is an image.
\begin{definition}[Strong Equivalence]
We say that two images $\mathcal{A}, \mathcal{B} \in G_{k\times m}(\mathbb{Z}_{n})(+,\cdot)$ are strongly equivalents if
$$
\mathcal{A}=\mathcal{S}+\mathcal{B},
$$
where $\mathcal{S}$ is a scalar image. We denote the strong equivalence between $\mathcal{A}$ and $\mathcal{B}$ as $\mathcal{A}\cong \mathcal{B}$.
\end{definition}
Note that if $\mathcal{A}=\mathcal{S}+\mathcal{B} \Rightarrow \exists\   \overline{\mathcal{S}}\ \vert\ \mathcal{B}=\overline{\mathcal{S}}+\mathcal{A}$ and $\overline{\mathcal{S}}=-(\mathcal{S})$, where $-(\mathcal{S})$ is the additive inverse of $\mathcal{S}$. This is calculated using the inverse of each pixels of $\mathcal{S}$ in $\mathbb{Z}_{n}$.
\begin{theorem}
If two images $\mathcal{A}$ and $\mathcal{B}$ are strongly equivalents then they are weakly equivalents.
\end{theorem}
\begin{proof}
If $\mathcal{A}$ and $\mathcal{B}$ are strongly equivalents then $\mathcal{A}=\mathcal{S}+\mathcal{B}$ where $\mathcal{S}$ is a scalar image. Then
$E(\mathcal{A}) =E(\mathcal{S}+\mathcal{B})$ but $\mathcal{S}$ is a scalar image and for this reason the sum $\mathcal{S}+\mathcal{B}$ only change in $\mathcal{B}$ the intensity of each pixel but don't change the number of different intensities or the frequency of each intensity in the image. Then, $E(\mathcal{S}+\mathcal{B}) =E(\mathcal{B})$. Finally we obtain that $E(\mathcal{A})=E(\mathcal{B})$ and they are weakly equivalents.\qed
\end{proof}
Note that the shown images in Figure \ref{different images comparation} are weakly equivalents, but they are not strongly equivalents. This is an example that in general $\mathcal{A}\asymp \mathcal{B} \nRightarrow \mathcal{A}\cong \mathcal{B}$.
\begin{definition}[Natural Entropy Distance]
\label{definition new stop}
Let $\mathcal{A}$ and $\mathcal{B}$ two images, then the natural entropy distance is defined by
\begin{equation}
\label{definition equation new}
\hat{\nu}(\mathcal{A},\mathcal{B})=E(\mathcal{A} + (-\mathcal{B})).
\end{equation}
\end{definition} 
\begin{remark}
Remember that $-(\mathcal{B})$ is the additive inverse of $\mathcal{B}$ and this is calculated using the inverse of each pixel of $\mathcal{B}$ in $\mathbb{Z}_{n}$.
\end{remark}
If it are considered the images of Figure \ref{different images comparation}, the results show that $\hat{\nu}(Figure\ \ref{one_a},\ Figure\ \ref{two_b})=0.2514$. This is more reasonable result.\\
%Note that really in the Definition \ref{definition new stop} the equation (\ref{definition equation new}) don't fulfill the triangle inequality, so, thus it is not proper distance metric.\\
The next theorem is an important characterization of the strong equivalent among images.
\begin{theorem}
Two images $\mathcal{A}$ and $\mathcal{B}$ are strongly equivalent if and only if $\hat{\nu}(\mathcal{A},\mathcal{B})=0$. 
\end{theorem}
\begin{proof}
If $\mathcal{A}$ and $\mathcal{B}$ are strongly equivalents $\mathcal{A}=\mathcal{S}+\mathcal{B}$ where $\mathcal{S}$ is the scalar image. Then we have
\begin{align*}
\hat{\nu}(\mathcal{A},\mathcal{B}) &= E(\mathcal{A} + (-\mathcal{B})) & \mbox{replacing }\quad \mathcal{A}=\mathcal{S}+\mathcal{B}\\
&= E(\mathcal{S}+\mathcal{B} + (-\mathcal{B})) & \\
&= E(\mathcal{S})=0. 
\end{align*} 
$E(\mathcal{S})=0$ because $\mathcal{S}$ is a scalar image. It is demonstrated that $\mathcal{A} \cong \mathcal{B} \Rightarrow \hat{\nu}(\mathcal{A},\mathcal{B})=0$.\\
On the other hand if $\hat{\nu}(\mathcal{A},\mathcal{B})=0 \Rightarrow \mathcal{A}+(-\mathcal{B})=\mathcal{S}$, where $\mathcal{S}$ is a scalar image. Adding $\mathcal{B}$ in the last equation we obtain that $\mathcal{A}=\mathcal{S}+\mathcal{B}$, therefore $\hat{\nu}(\mathcal{A},\mathcal{B})=0 \Rightarrow \mathcal{A} \cong \mathcal{B}$.  \qed  
\end{proof}
Taking in consideration  the good properties that, in general, the natural entropy distance has (see Definition \ref{definition new stop}), one sees logical to take the condition (\ref{definition equation new}) as the new stopping criterion of the iterative mean shift algorithm. Explicitly, the new stopping criterion is 
\begin{equation}
\label{definition stop criterion new}
E(\mathcal{A}_{k}+(-\mathcal{A}_{k-1}))\leq \epsilon ,
\end{equation}
where $\epsilon$ and $k$ are respectively the threshold to stop the iterations and the number of iterations.
\section{Experiments and Results}
\label{Experiments and Results}
Image segmentation, that is, classification of the image gray-level values into homogeneous areas is recognized to be one of the most important step in any image analysis system. Homogeneity, in general, is defined as similarity among the pixel values, where a piecewise constant model is enforced over the image \cite{Comaniciu02}. \\
The principal goal of this section is to evaluate the new stop criterion in the iterative mean shift algorithm and to prove that, in general, with this new stopping criterion the algorithm have better stability. For this aim, we used three different images for the experiments. The first image (``	Bird") have low frequency, the second (``Baboon") have high frequency and in the image ``Montage" has mixture low and high frequencies. \\
All segmentation experiments were carried out by using a uniform kernel. In order to be effective the comparison between the old stopping criterion and the new stopping criterion, we use the same value of $hr$ and $hs$ in the iterative mean shift algorithm ($hr=12,\ hs=15 $). The value of $hs$ is related to the spatial resolution of the analysis, while the value $hr$ defines the range resolution. In the case of the new stopping criterion, we use the stopping threshold $\epsilon= 0.9$ and when the old stopping criterion was used, we selected $\epsilon= 0.0175$.\\
Figure \ref{test images} shows the segmentation of the three images. Observe that, in all cases, the iterative mean shift algorithm had better result when was used the new stopping criterion.\\ 
When one compares Figures \ref{bird_new} and \ref{bird_old}, in the part corresponding to the face or breast of the bird a more homogeneous area, with the new stopping criterion (see arrows in Figure \ref{bird_old}), it was obtained. Observe that, with the old stopping criterion the segmentation gives regions where different gray levels are originated. However, these regions really should have only one gray level. For example, Figure \ref{baboon_new} and \ref{baboon_old} show that the segmentation is more homogeneous when the new stopping criterion was used (see the arrows). In the case of the ``Montage" image one can see that, in Figure \ref{montage_old} exists many regions that contains different gray levels when these regions really should have one gray level (see for example the face of Lenna, the circles and the breast of the bird). These good results are obtained because the defined new stopping criterion  through the natural distance between images in expression (\ref{definition stop criterion new}) offers greater stability to the mean shift iterative algorithm.

\begin{figure}
\centering
\subfigure[Bird]{\includegraphics[scale=0.4]{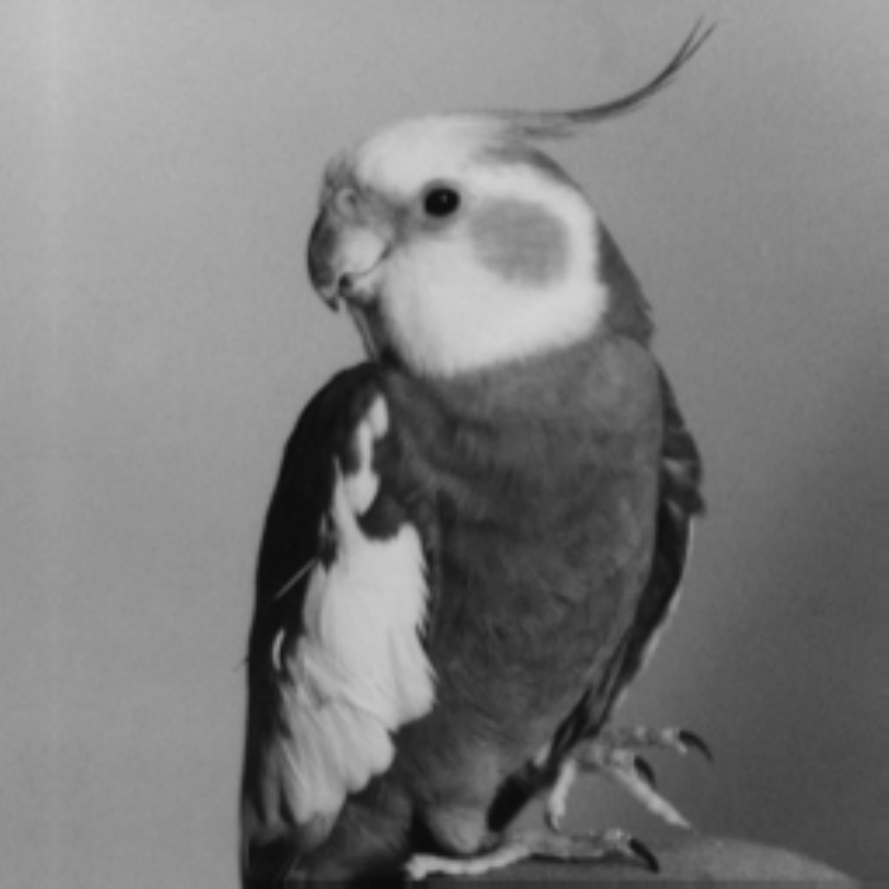}}
\subfigure[New Criterion]{\includegraphics[scale=0.4]{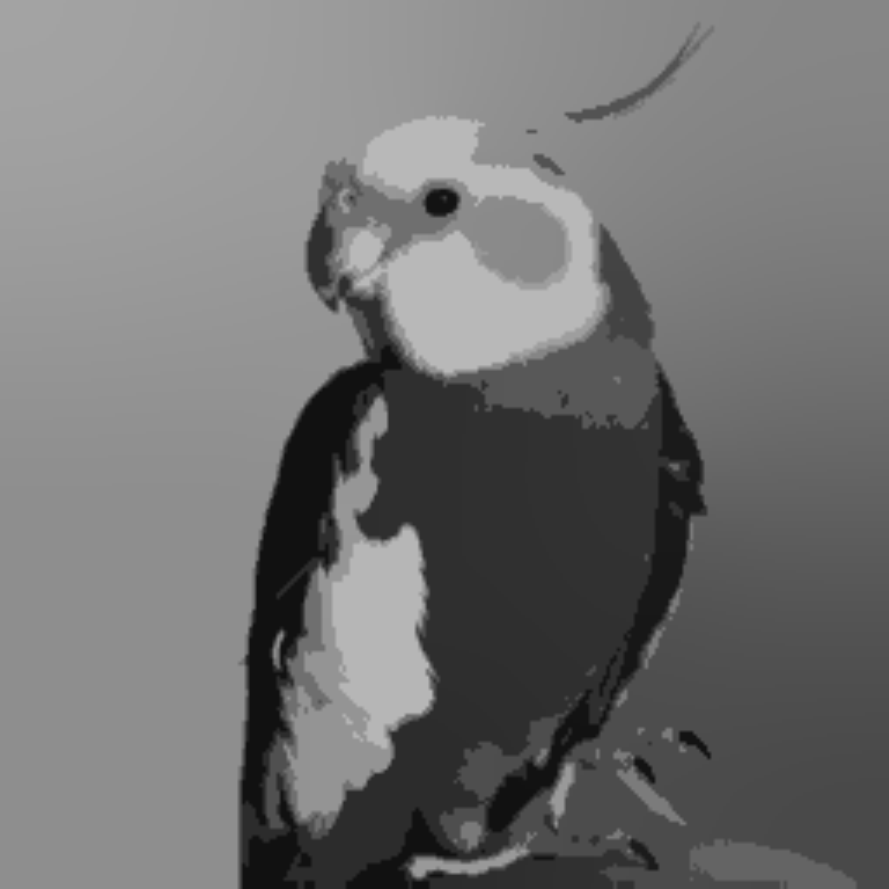}
\label{bird_new}}
\subfigure[Old Criterion]{\includegraphics[scale=0.4]{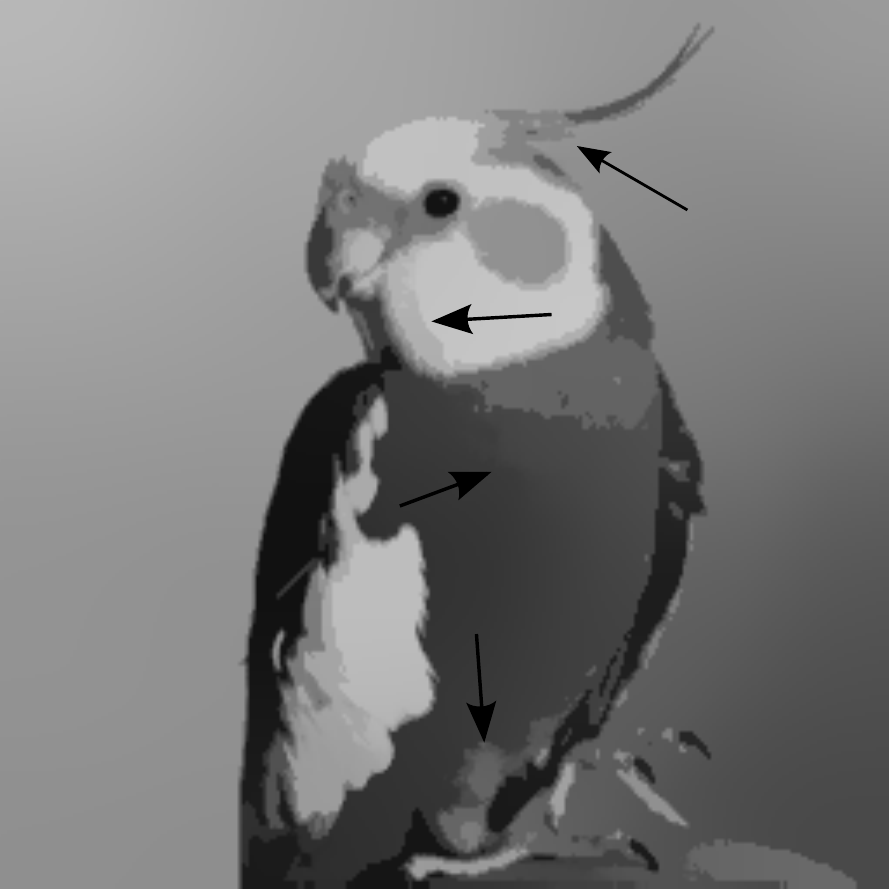}
\label{bird_old} }\\
\vspace{-0.1in}
\subfigure[Baboon]{\includegraphics[scale=0.4]{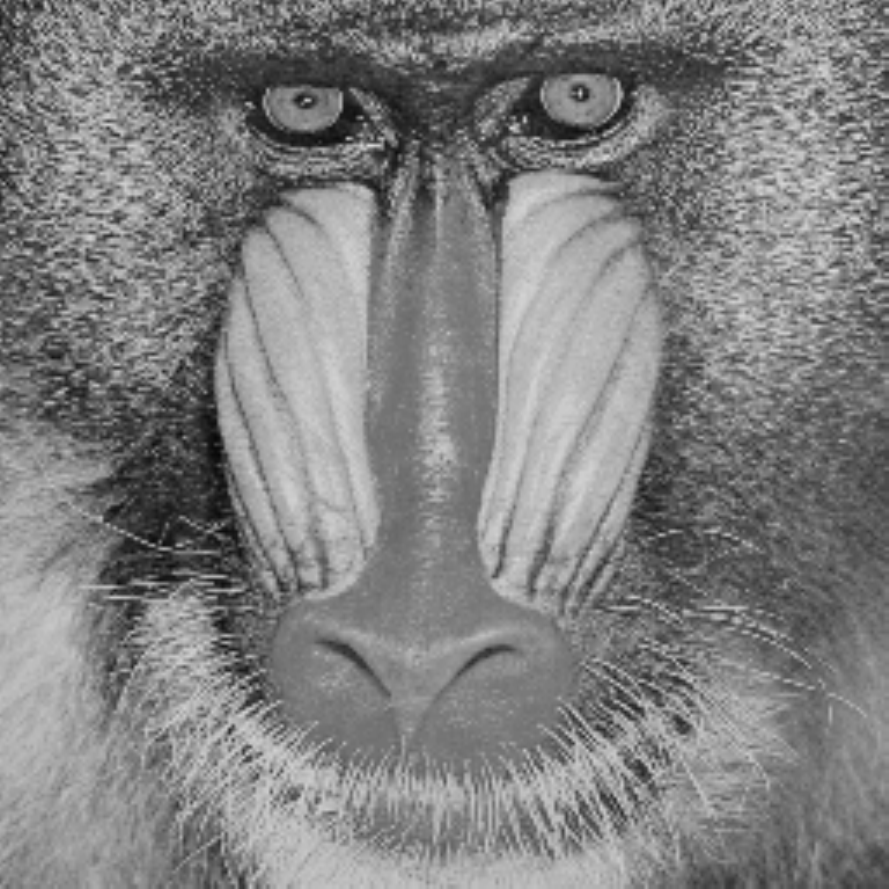}
\label{baboon}}
\subfigure[New Criterion]{\includegraphics[scale=0.4]{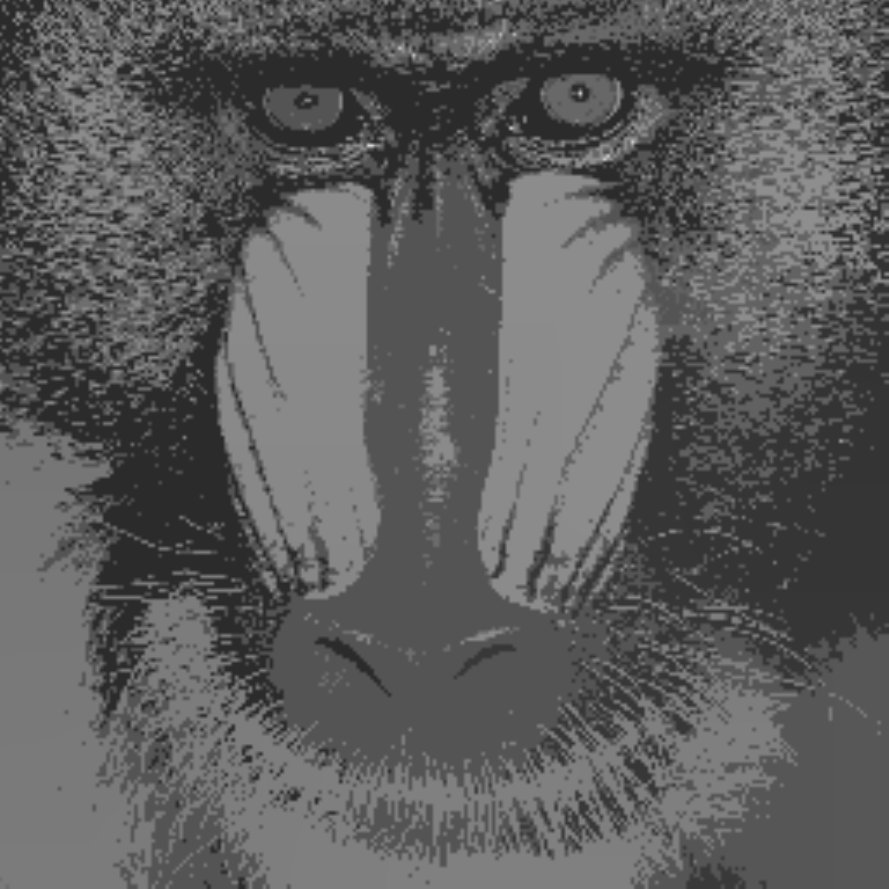}
\label{baboon_new}}
\subfigure[Old Criterion]{\includegraphics[scale=0.4]{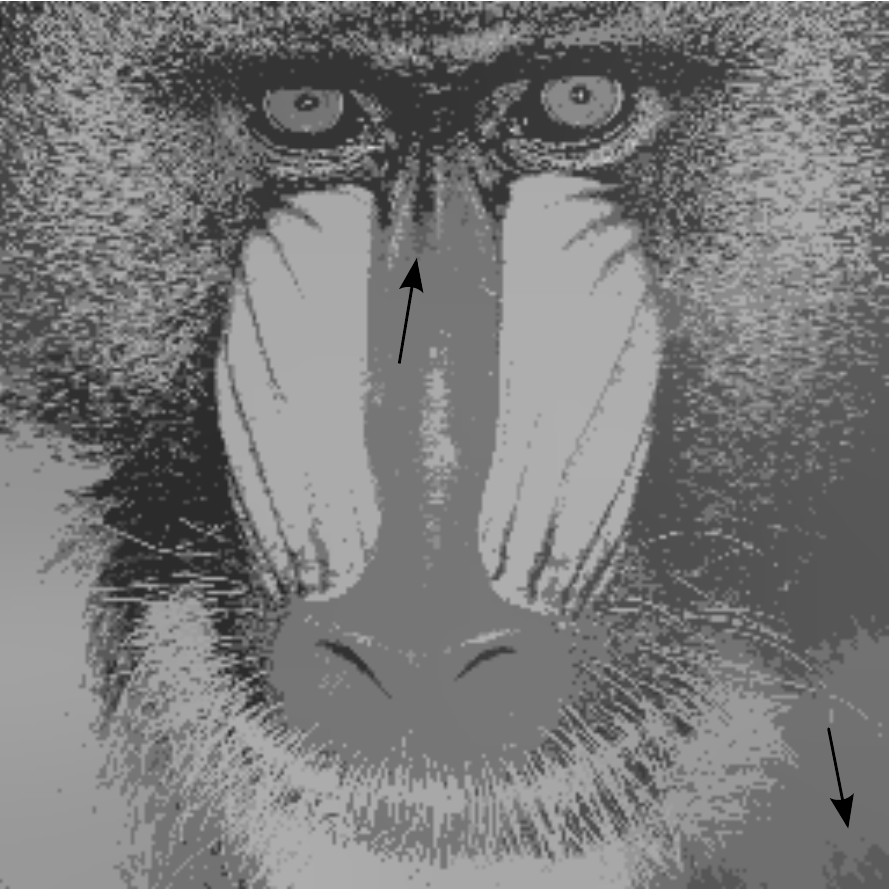}
\label{baboon_old}}\\
\vspace{-0.1in}
\subfigure[Montage]{\includegraphics[scale=0.4]{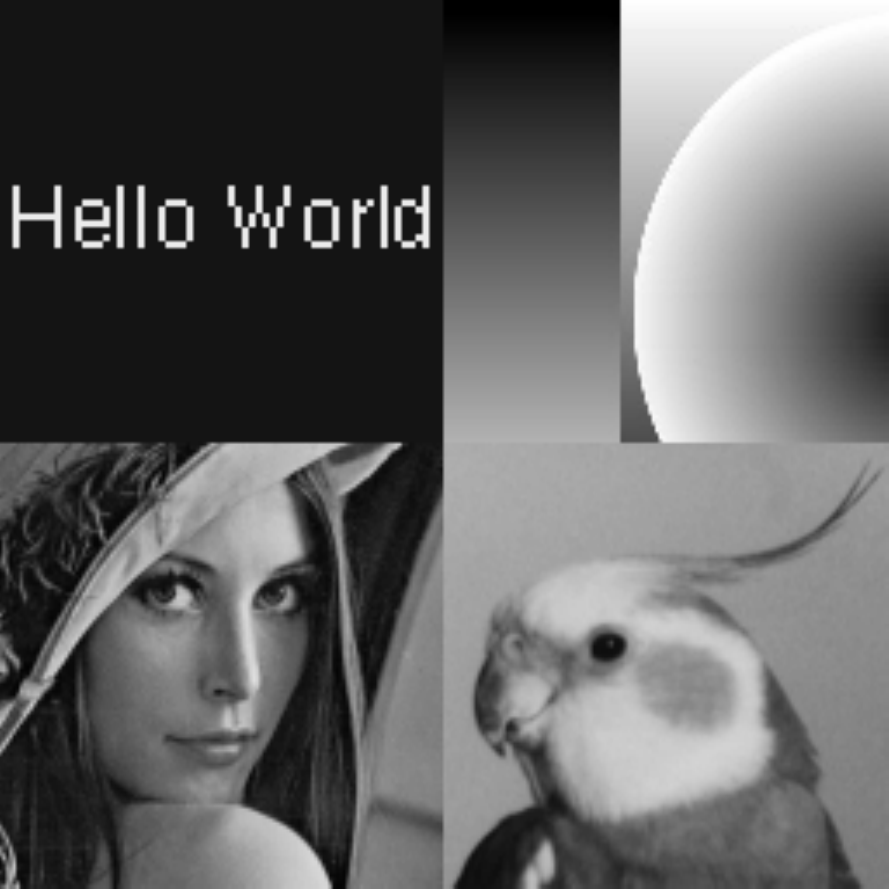}}
\subfigure[New Criterion]{\includegraphics[scale=0.4]{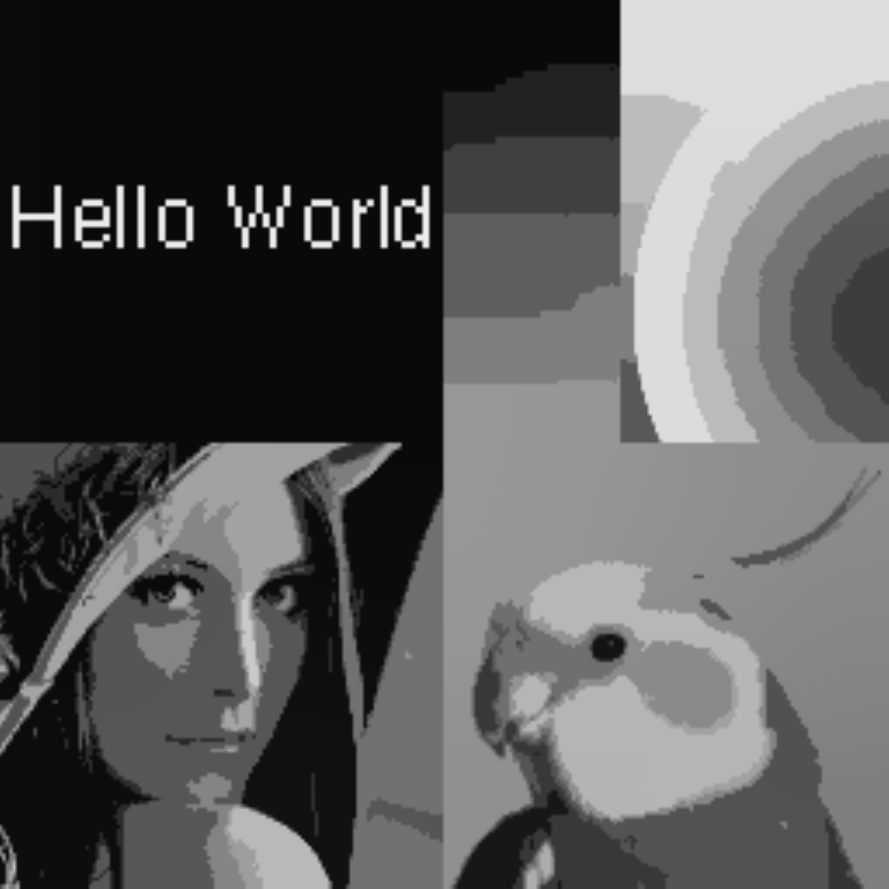}
\label{montage_new}}
\subfigure[Old Criterion]{\includegraphics[scale=0.4]{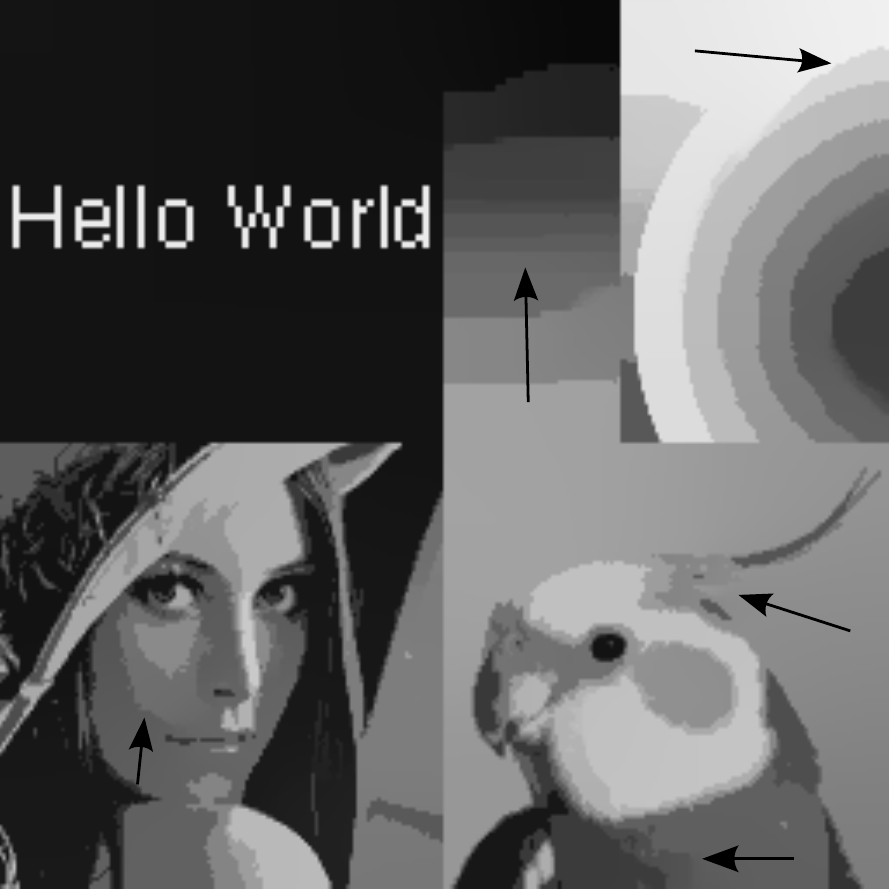}
\label{montage_old}}
\vspace{-.1in}
\caption{Segmentation of the experimental images. In the first column are show the original images; in the second, the segmentation using the new stopping criterion and in the third column are the segmented images using the old stopping criterion.}
\label{test images}
\vspace{-.2in}
\end{figure}

Figure \ref{profile test} shows the profile of the obtained segmented images by using the two stopping criterion\footnote{We show only the profile of one image for reasons of space, but the results in the other images were similar.}. The plates that appear in Figure \ref{profile new} and \ref{profile old} are indicative of equal intensity levels. In both graphics the abrupt falls of an intensity to other represent the different regions in the segmented image. Note that, in Figure \ref{profile new} exists, in the same region of the segmentation, least variation of the pixel intensities with regard to Figure \ref{profile old}. This illustrates that, in this case the segmentation was better when the new stopping criterion was used. 

\begin{figure}
\centering
\subfigure[New Criterion]{\includegraphics[scale=0.3]{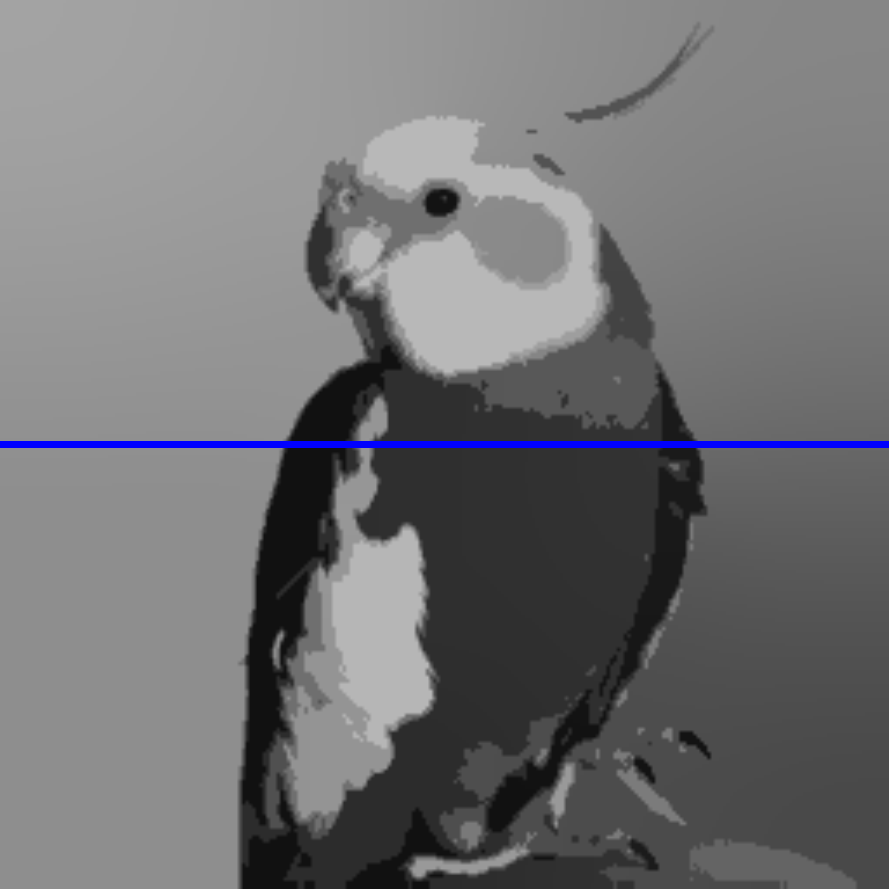} }
\subfigure[Profile]{\includegraphics[scale=0.45]{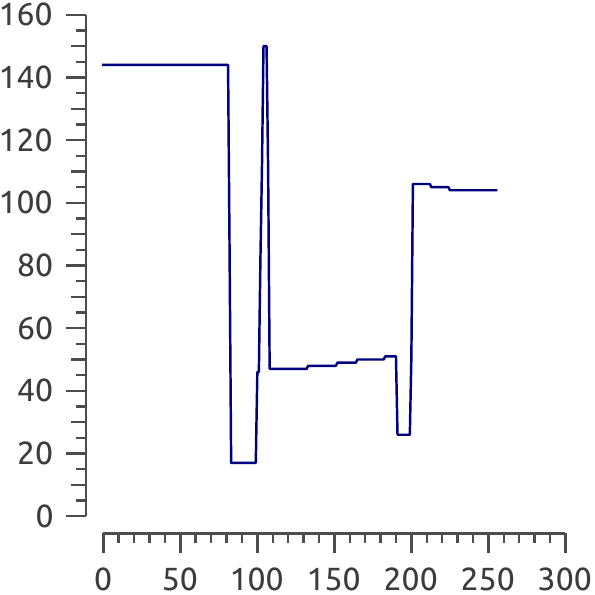}
\label{profile new} }
\subfigure[Old Criterion]{\includegraphics[scale=0.3]{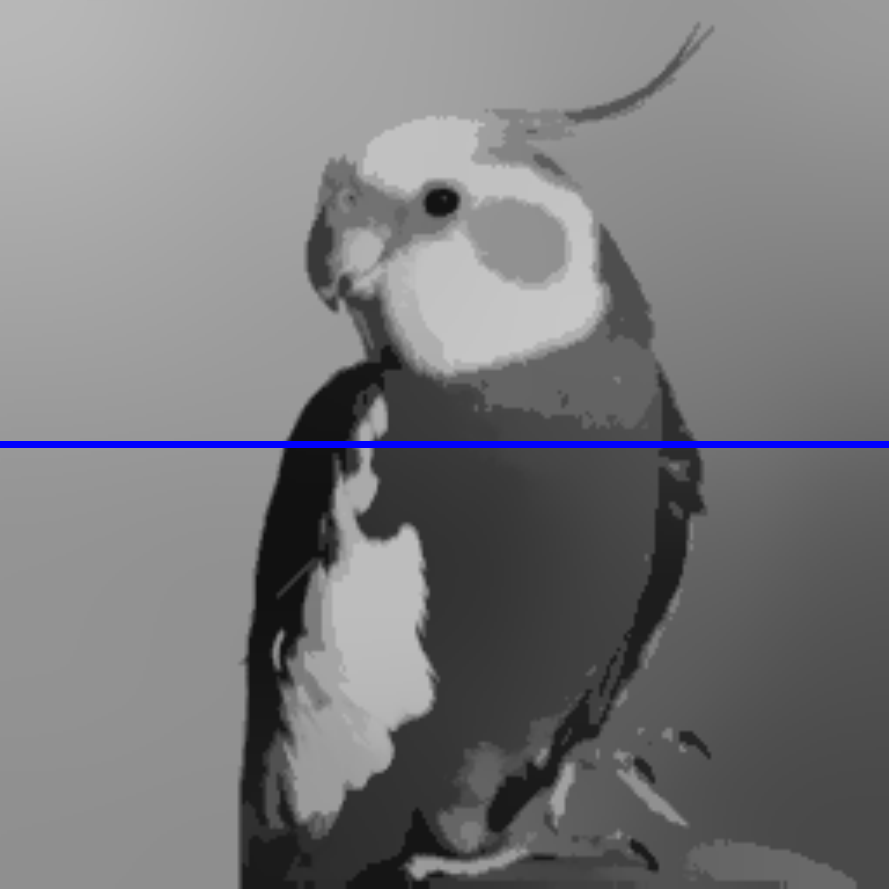} }
\subfigure[Profile]{\includegraphics[scale=0.45]{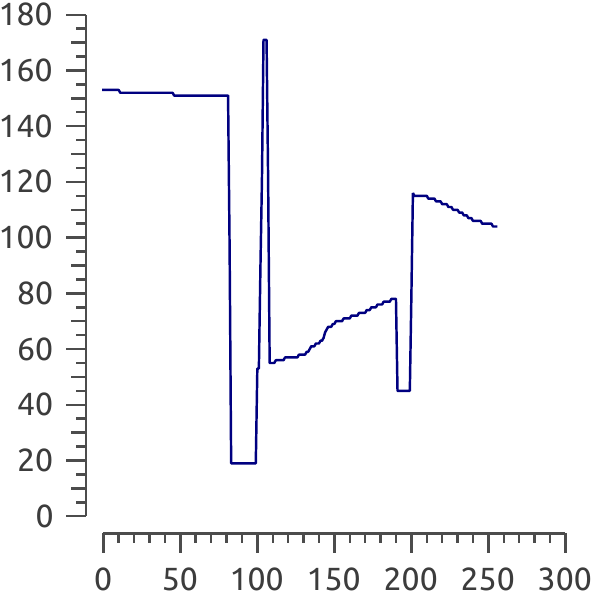} 
\label{profile old} }
\vspace{-.1in}
\caption{An intensity profile through of a segmented image. Profile is indicated by a line. (a) and (c) are the segmented images and (b) and (d) are the profile of (a) and (c) respectively.}
\label{profile test}
\vspace{-.2in}
\end{figure}
Figure \ref{iterations} shows the performance of the two stopping criterion in the experimental images. In the ``$x$" axis appears the iterations of the mean shift algorithm and in the ``$y$" axis is shown the obtained values by the stopping criterion in each iteration of the algorithm.\\ 
The graphics of iterations of the new stopping criterion (Figure \ref{iterations_new_bird}, \ref{iterations_new_baboon}, \ref{iterations_new_montage}) show a smooth behavior; that is, the stopping criterion has a stable performance through the iterative mean shift algorithm. The new stopping criterion not only has good theoretical properties, but also, in the practice, has very good behavior.\\
On the other hand, if we analyze the old stopping criterion in the experimental images (Figure \ref{iterations_old_bird}, \ref{iterations_old_baboon}, \ref{iterations_old_montage}), we can see that the performance in the mean shift algorithm is unstable. In general, we have this type of situation when the stopping criterion defined in (\ref{old criterion}) is used. This can originate bad segmented images.\\
\vspace{-.25in}
\begin{figure}
\centering
\subfigure[Bird]{\includegraphics[width=3.5cm, height=2.8cm]{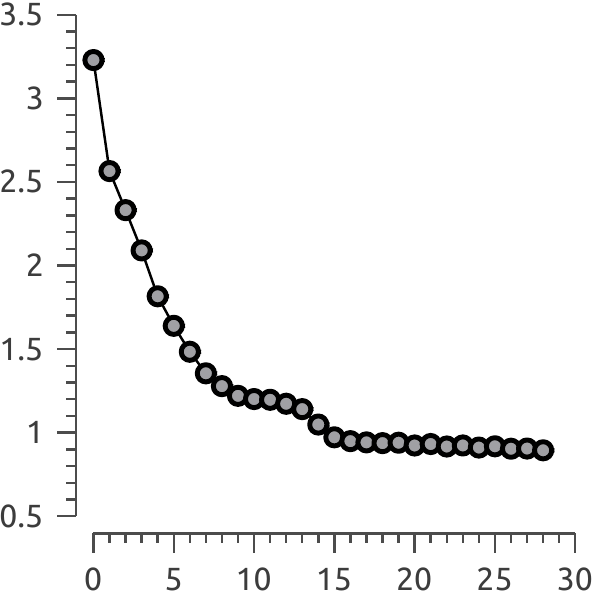}
\label{iterations_new_bird} }
\subfigure[Baboon]{\includegraphics[width=3.5cm, height=2.8cm]{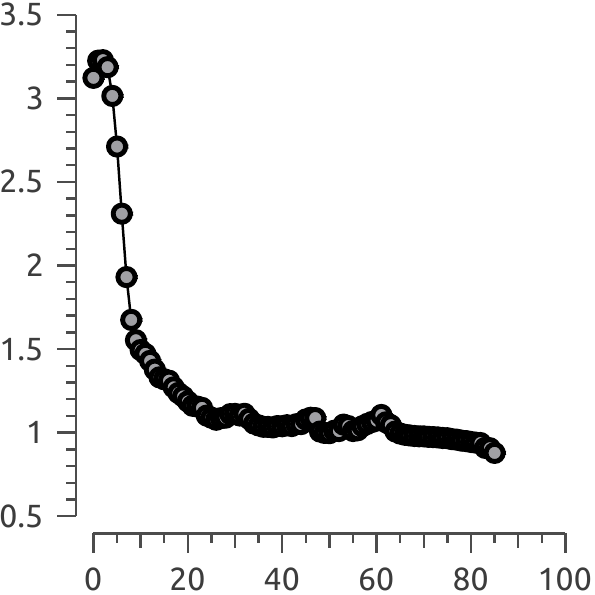} 
\label{iterations_new_baboon}}
\subfigure[Montage]{\includegraphics[width=3.5cm, height=2.8cm]{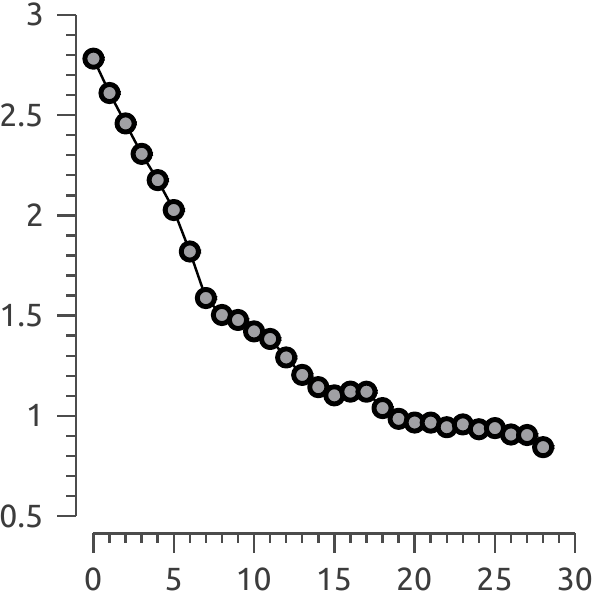} 
\label{iterations_new_montage}}\\
\vspace{-0.1in}
\subfigure[Bird]{\includegraphics[width=3.5cm, height=2.8cm]{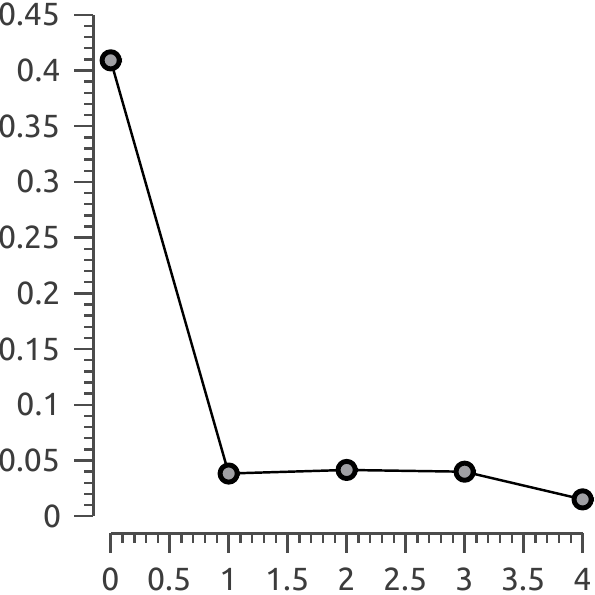} 
\label{iterations_old_bird}}
\subfigure[Baboon]{\includegraphics[width=3.5cm, height=2.8cm]{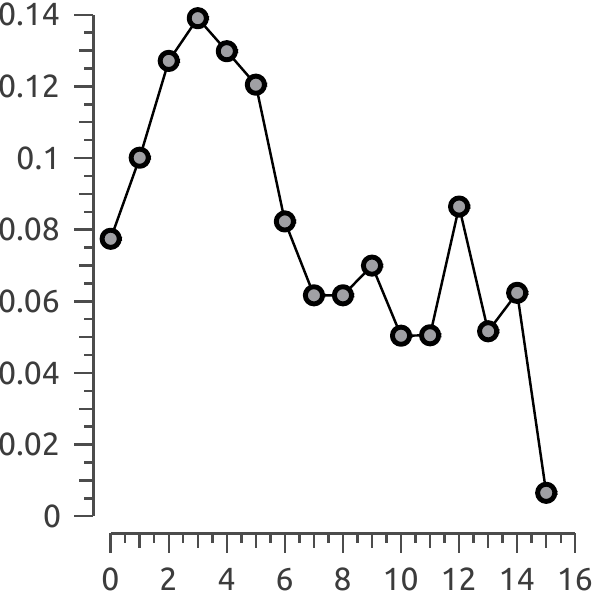} 
\label{iterations_old_baboon}}
\subfigure[Montage]{\includegraphics[width=3.5cm, height=2.8cm]{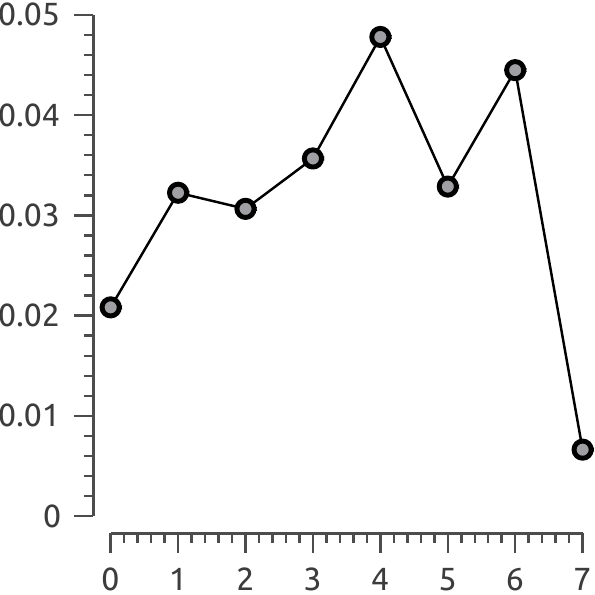} 
\label{iterations_old_montage}}
\vspace{-.15in}
\caption{Stopping criterion. In the first row appears the performance of the new stopping criterion and in the second, it is shown the old stopping criterion in correspondence with the experimental images.}
\label{iterations}
\vspace{-.2in}
\end{figure}
\vspace{-.2in}
\section*{Conclusions}
\vspace{-.1in}
In this work, a new stopping criterion, for the iterative mean shift algorithm, based on the ring theory was proposed. The new stopping criterion establishes a new measure for the comparison of two images based on the use of the entropy concept. We introduced a new way to operate with images based on the use of the ring structure. The rings in the images space were defined using the concept of $\mathbb{Z}_{n}$ rings. Through the obtained theoretical and practical results, it was possible to prove that the new stopping criterion had very good performance in the iterative mean shift algorithm, and in general, it was more stable that the old criterion \cite{Rodriguez11,Rodriguez11a,Rodriguez12,Rodriguez08}.\\

\end{document}